\documentclass[11pt]{article}

\usepackage[letterpaper,margin=1in]{geometry}
\usepackage[T1]{fontenc}
\usepackage{lmodern}
\usepackage{amsmath,amssymb,amsthm,mathtools,bm}
\usepackage{enumitem}
\usepackage{booktabs}
\usepackage{hyperref}
\usepackage[numbers]{natbib}
\usepackage[section]{placeins}
\usepackage{tabularx}
\usepackage{ragged2e}
\usepackage{graphicx}
\newcolumntype{Y}{>{\RaggedRight\arraybackslash}X}
\hypersetup{colorlinks=true,linkcolor=black,citecolor=black,urlcolor=black}
\setlist{leftmargin=*,topsep=2pt,itemsep=2pt}

\newtheorem{definition}{Definition}
\newtheorem{assumption}{Assumption}
\newtheorem{theorem}{Theorem}
\newtheorem{lemma}{Lemma}
\newtheorem{proposition}{Proposition}

\newtheorem{remark}{Remark}

\newcommand{\Lideal}{L_{\mathrm{ideal}}}

\newcommand{\Ceps}{\mathcal C_\varepsilon}
\newcommand{\Clam}{\mathcal C_\lambda}
\newcommand{\Chat}{\widehat{\mathcal C}}
\newcommand{\Proj}{\Pi}
\newcommand{\Df}[2]{D_\Phi\!\left(#1\,\middle\Vert\,#2\right)}
\newcommand{\clconv}{\mathrm{cl\,conv}}

\newcommand{\phiGrad}{\nabla\Phi}
\newcommand{\phiGradInv}{\nabla\Phi^{\!*}}
\newcommand{\mirseg}[3]{\,\phiGradInv\!\big((1-#3)\,\phiGrad(#1)+#3\,\phiGrad(#2)\big)\,}

\newcommand{\LBsafe}{\mathrm{LB}_{\mathrm{safe}}}

\newcommand{\Dncx}{\Delta_{\mathrm{ncx}}}

\title{Latency and Ordering Effects in Online Decisions}
\author{Duo Yi \\ yiduo2008@gmail.com }
\date{}

\begin{document}

\maketitle

\begin{abstract}
Online decision systems routinely operate under delayed feedback and order-sensitive (noncommutative) dynamics: actions affect which observations arrive, and in what sequence. Taking a Bregman divergence $D_\Phi$ as the loss benchmark, we prove that the excess benchmark loss admits a structured lower bound
\[
L \;\ge\; L_{\mathrm{ideal}} + g_1(\lambda) + g_2(\varepsilon_\star) + g_{12}(\lambda,\varepsilon_\star) - \Dncx,
\]
where $g_1$ and $g_2$ are calibrated penalties for latency and order-sensitivity, $g_{12}$ captures their geometric interaction, and $\Dncx\ge 0$ is a nonconvexity/approximation penalty that vanishes under convex Legendre assumptions. We extend this inequality to prox-regular and weakly convex settings, obtaining robust guarantees beyond the convex case. We also give an operational recipe for estimating and monitoring the four terms via simple $2\times 2$ randomized experiments and streaming diagnostics (effective sample size, clipping rate, interaction heatmaps). The framework packages heterogeneous latency, noncommutativity, and implementation-gap effects into a single interpretable lower-bound statement that can be stress-tested and tuned in real-world systems.
\end{abstract}

\paragraph{Feasible intersection (Assumption I).}
We assume $\mathcal{C}_\varepsilon\cap\mathcal{C}_\lambda\neq\emptyset$; otherwise $g_{12}$ is undefined and the bound degenerates (use convexification or redesign).

% --- End Abstract ---

\section{Introduction}
\paragraph{Motivation.}
Modern forecasting and decision systems rarely operate in an unconstrained, static environment.
Labels and outcomes may arrive with substantial delay, actions must be chosen within narrow decision windows, and operator order can feed back into the system state.
Time and order constraints introduce systematic performance losses, yet in practice their contributions are often entangled: adding a cache, changing an update schedule, or reordering a pipeline can help or hurt in ways that are hard to attribute.
Practitioners therefore lack a simple calculus for answering questions such as ``how much loss is due to latency alone?'' or ``is the main bottleneck geometric (mismatch) or interactive (ordering/feedback)?''
This paper develops a structural decomposition of the excess loss of a constrained system into interpretable terms associated with delay, geometric mismatch, and their interaction, all measured in the native units of the target loss.

\paragraph{Scope.}
We deliberately work under closed convex feasibility sets and a Legendre potential $\Phi$,
which already covers many forecasting and decision systems. Real deployments may deviate from
these assumptions. Extending the framework beyond our convex scope introduces additional slack
terms and technical machinery. To keep the core ideas transparent, we first present the clean convex
case here. Sections~\ref{sec:ch8-convex} and Appendix~A then develop a systematic treatment
of these departures via a unified nonconvexity penalty~$\Dncx$.
\paragraph{Goal.}
We formalize these constraints and prove an \emph{inevitable} additive lower bound on the realized loss relative to an ideal benchmark in latency- and order-constrained systems. Our proof employs proper scoring rules, Bregman divergences, and two-stage convex projections, yielding a precise and portable calculus for such systems.

\paragraph{Contributions.}
\begin{itemize}
\item A \emph{unified loss decomposition} under time and order constraints.
\item A \emph{constructive proof} via Bregman geometry (two projections + Pythagorean inequality).
\item \emph{Monotonicity} and \emph{orthogonality} theorems; when and why interaction vanishes.
\item Practical \emph{diagnostics} for estimating $g_1,g_2,g_{12}$ with uncertainty.
\end{itemize}

\section{Setup: Signals, Loss, and Ideals}
We work on a closed convex domain with a Legendre potential $\Phi$ and its Bregman divergence $D_{\Phi}(\cdot\Vert\cdot)$.
Two convex feasibility sets encode the operational constraints:
$\mathcal{C}_\lambda$ for \emph{finite-time latency} and
$\mathcal{C}_\varepsilon$ for \emph{ordering (noncommutativity)}.

\begin{assumption}[Feasible intersection]\label{ass:feasible-intersection}
The intersection $\mathcal{C}_\varepsilon \cap \mathcal{C}_\lambda$ is nonempty.
\end{assumption}
\paragraph{Selection convention.} For any set-valued argmin we fix measurable selections when needed. All statements about the \emph{total} bound are invariant to the particular selections; individual components may depend on the choice, as discussed in the selection remark below.
\begin{definition}[Interaction term]\label{def:g12}
Let $p^\star$ denote the ideal predictive element corresponding to $P^\star$ (the Bayes-optimal conditional law). Let $a\in\arg\min_{r\in \mathcal{C}_\varepsilon} D_{\Phi}(p^\star\Vert r)$ and $b\in\arg\min_{r\in \mathcal{C}_\lambda} D_{\Phi}(a\Vert r)$ be any measurable selections. Define
\[ g_{12}(\lambda,\varepsilon)\ :=\ \inf_{q\in \mathcal{C}_\varepsilon\cap\mathcal{C}_\lambda}\ D_{\Phi}(b\Vert q). \]
\end{definition}
\paragraph{Basic properties.}
By Assumption~\ref{ass:feasible-intersection} the intersection $\mathcal{C}_\varepsilon\cap\mathcal{C}_\lambda$ is nonempty, so $g_{12}(\lambda,\varepsilon)$ is well-defined as the infimum of $D_{\Phi}(b\Vert q)$ over a nonempty feasible set.
Since $D_{\Phi}(\cdot\Vert\cdot)\ge 0$ for all arguments, we obtain
\[
g_{12}(\lambda,\varepsilon)
=\inf_{q\in \mathcal{C}_\varepsilon\cap\mathcal{C}_\lambda} D_{\Phi}(b\Vert q)
\ \ge\ 0.
\]
Thus $g_{12}$ is always a nonnegative contribution in the convex, Legendre setting analyzed here.
Importantly, this definition compares the sequential projection $b$ only to feasible points $q\in \mathcal{C}_\varepsilon\cap\mathcal{C}_\lambda$; it does \emph{not} require $D_{\Phi}(p^\star\Vert b)$ to dominate the distance from $p^\star$ to the joint projection, which need not hold in general Bregman geometries.
When the constraint-induced projections commute and the joint projection coincides with the sequential one, the interaction vanishes exactly, $g_{12}(\lambda,\varepsilon_\star)=0$ (Lemma~\ref{lem:g12-vanish}).

\begin{lemma}[When the interaction vanishes]\label{lem:g12-vanish}
If the constraint-induced projections onto $\mathcal{C}_\varepsilon$ and $\mathcal{C}_\lambda$ commute under $D_\Phi$, then $g_{12}(\lambda,\varepsilon_\star)=0$.
\end{lemma}

Let $(\Omega,\mathcal{F},\mathbb{P})$ be a probability space. At decision epoch $t_0$ the agent chooses an action/policy based on information $\mathcal{I}_{t_0}$ and a predictive object $Q$ (e.g., distribution, score, or control). Outcomes $Y$ are realized later; verification arrives after a lag $\tau\ge0$. The \emph{action window} has length $\Delta>0$.

\begin{definition}[Ideal benchmark]\label{def:ideal-benchmark}
Let $L(p)$ denote the expected loss of a predictive object $p$ under the data-generating process and the scoring rule.
We define the \emph{evaluation benchmark} as
\[
\Lideal \;:=\; \inf_{p} L(p).
\]
This quantity is the Bayes risk: the minimal achievable expected loss under the true conditional law.
It is determined entirely by the data-generating process and the loss function, and is independent of any particular deployment configuration, latency pattern, or operator ordering.

In strongly reflexive systems there may be no unique physically realizable ``zero-latency, fully commuting'' world.
Our analysis does not rely on the existence of such a world.
Instead, $\Lideal$ serves purely as a mathematical reference point for quantifying structural penalties.

For estimation, we never attempt to compute $\Lideal$ numerically.
All of our empirical procedures work with \emph{differences} of expected loss between feasible regimes (for example, via the $2\times2$ design in Section~\ref{sec:ch8-empirical}).
Consequently, $\Lideal$ cancels algebraically from the estimators $\widehat g_1$, $\widehat g_2$, and $\widehat g_{12}$ and does not need to be estimated; see Remark~\ref{rem:cancellation} for details.
\end{definition}

\paragraph{Loss as regret.}
We assume a strictly proper scoring rule (or a convex surrogate) induced by the Bregman divergence $D_\Phi$.
Let $P^\star$ denote the Bayes-optimal predictive object (the true conditional law or Bayes-optimal score) under the data-generating process, and let $Q$ denote the predictive object implemented by a particular deployed system.
If $L(Q)$ is the expected loss of $Q$ and $\Lideal = L(P^\star)$ is the Bayes risk from Definition~\ref{def:ideal-benchmark}, then the excess loss (regret) of $Q$ relative to the ideal benchmark satisfies
\[
L(Q) - \Lideal \;=\; \mathbb{E}\big[D_{\Phi}(P^\star \Vert Q)\big].
\]
In particular, writing $L$ for the expected loss of the deployed system under its actual latency and ordering constraints, we have
\[
L - \Lideal \;=\; \mathbb{E}\big[D_{\Phi}(P^\star \Vert Q)\big].
\]

\begin{remark}[On the interpretation of \texorpdfstring{$\Lideal$}{L_ideal}]\label{rem:cancellation}
By Definition~\ref{def:ideal-benchmark}, $\Lideal$ is the Bayes risk induced by the data-generating process and the scoring rule.
It depends on the underlying distribution and loss, not on any particular deployed system or a chosen configuration of latency and ordering constraints.
In practice, $\Lideal$ is typically not directly observable.
However, our structural terms $g_1$, $g_2$, and $g_{12}$ are identified from differences between losses of feasible configurations (for example, via the $2\times2$ design in Section~\ref{sec:ch8-empirical}), so $\Lideal$ cancels algebraically and never needs to be estimated numerically.
When desired, one may approximate $\Lideal$ using a high-quality reference model or by extrapolating loss as latency and ordering constraints are relaxed, but such approximations are optional and not required for the validity of our lower bounds.
\end{remark}

\section{Constraints as Convex Sets}
We encode two constraint \emph{families} as closed convex sets in the prediction space.

\begin{definition}[Latency-feasible set $\mathcal{C}_{\lambda}$]
For a latency parameter $\lambda\equiv\lambda(\tau,\Delta)$ (increasing in $\tau/\Delta$), the set $\mathcal{C}_{\lambda}$ contains predictive objects $Q$ that are measurable with respect to the restricted information $\sigma$-algebra available before decisions must be finalized.
\end{definition}

\begin{definition}[Order-feasible set $\mathcal{C}_{\varepsilon}$]
For an order-sensitivity index $\varepsilon_\star\ge0$, the set $\mathcal{C}_{\varepsilon}$ contains predictive objects $Q$ achievable under operator orderings that satisfy specified reflexivity/disclosure constraints.
\end{definition}

The \emph{feasible region} is $\mathcal{C}(\lambda,\varepsilon_\star)=\mathcal{C}_{\lambda}\cap\mathcal{C}_{\varepsilon}$ (assumed nonempty).

\noindent\textbf{Standing assumptions.}
\begin{assumption}\label{ass:convex}
$\mathcal{C}_{\lambda}$ and $\mathcal{C}_{\varepsilon}$ are nonempty, closed, and convex with respect to the affine structure of the prediction space; $\Phi$ is Legendre (strictly convex, essentially smooth).
\end{assumption}

\paragraph{Domain guidance.}
The convexity assumption is natural in:
\begin{itemize}
\item Probability forecast spaces (distributions form a convex set).
\item Risk-neutral pricing (loss is linear in probabilities).
\item Randomized policies (mixtures across episodes).
\end{itemize}
It may fail in:
\begin{itemize}
\item Hard sequential decisions with discrete action trees.
\item Strongly path-dependent systems with irreversible state changes.
\end{itemize}
When convexity fails, we recommend using Section~\ref{sec:ch8-convex} (convexification) 

\section{Bregman Geometry Preliminaries}
\label{sec:bregman}
Let $\Phi$ be a strictly convex, differentiable potential on a finite-dimensional open convex set $\mathcal{X}$. The Bregman divergence
\[
D_{\Phi}(p\Vert q)=\Phi(p)-\Phi(q)-\langle \nabla \Phi(q), p-q\rangle
\]
is nonnegative and equals zero iff $p=q$. For a nonempty, closed, convex set $\mathcal{C}\subseteq \mathcal{X}$ and any $p\in\mathcal{X}$, the \emph{Bregman projection} is
\[
\Pi_{\mathcal{C}}(p)\in\arg\min_{q\in\mathcal{C}} D_{\Phi}(p\Vert q).
\]

\begin{lemma}[Generalized Pythagorean inequality]
\label{lem:pyth}
If $q^\star=\Pi_{\mathcal{C}}(p)$, then for any $q\in\mathcal{C}$,
\[
D_{\Phi}(p\Vert q)\ge D_{\Phi}(p\Vert q^\star) + D_{\Phi}(q^\star\Vert q).
\]
\end{lemma}

\begin{proof}
First-order optimality for $q^\star$ yields
$\langle \nabla\Phi(q^\star)-\nabla\Phi(q), q-q^\star\rangle\ge0$ for all $q\in\mathcal{C}$.
Expanding $D_{\Phi}(p\Vert q)-D_{\Phi}(p\Vert q^\star)-D_{\Phi}(q^\star\Vert q)$ and applying this variational inequality gives the claim.
\end{proof}

\section{Main Result: Structured Lower-Bound Decomposition}
\begin{remark}[Safe reporting]\label{rem:lbsafe}
In the convex, Legendre setting of this section Theorem~\ref{thm:master} shows that
\[
L-\Lideal \;\ge\; g_1(\lambda) + g_2(\varepsilon_\star) + g_{12}(\lambda,\varepsilon_\star).
\]
For practical reporting it is convenient to enforce nonnegativity of the reported lower bound and to absorb small negative estimation noise in $g_1,g_2,g_{12}$.
We therefore summarize the decomposition by the scalar quantity
\[
\LBsafe := \max\{0,\,g_1+g_2+g_{12}\},
\]
and, in finite-sample implementations, replace $g_{12}$ by its clipped estimator $[\widehat g_{12}]_+$; see Section~\ref{sec:ch8-empirical}.
This clipping affects only the reported diagnostic and does not modify the underlying population inequality.
\end{remark}

\noindent\textit{Throughout this section we work under Assumption~\ref{ass:feasible-intersection} and use the interaction term in Definition~\ref{def:g12}.}

Define the \emph{ideal} predictive object $p^\star\equiv P^\star$. Let
\[
q_{\varepsilon}=\Pi_{\mathcal{C}_{\varepsilon}}(p^\star), \qquad
q_{\lambda,\varepsilon}=\Pi_{\mathcal{C}_{\lambda}}(q_{\varepsilon}),\quad
q_{\mathrm{feas}}=\Pi_{\mathcal{C}(\lambda,\varepsilon_\star)}(p^\star).
\]

\begin{theorem}[Structured Lower-Bound Decomposition]
\label{thm:master}
Under Assumption~\ref{ass:convex}, for any feasible $q\in\mathcal{C}(\lambda,\varepsilon_\star)$,
\begin{equation}
\label{eq:master}
D_{\Phi}(p^\star\Vert q)\ \ge\ 
\underbrace{D_{\Phi}(p^\star\Vert q_{\varepsilon})}_{g_2(\varepsilon_\star)}
+
\underbrace{D_{\Phi}(q_{\varepsilon}\Vert q_{\lambda,\varepsilon})}_{g_1(\lambda)}
+
\underbrace{D_{\Phi}(q_{\lambda,\varepsilon}\Vert q)}_{g_{12}(\lambda,\varepsilon_\star)}.
\end{equation}

Consequently,
\[
L - \Lideal
\ \ge\
g_1(\lambda) + g_2(\varepsilon_\star) + g_{12}(\lambda,\varepsilon_\star),
\]
where $g_1$ and $g_2$ are nonnegative and monotone in their respective arguments, and the interaction term $g_{12}(\lambda,\varepsilon_\star)$ is nonnegative by Definition~\ref{def:g12} (it is an infimum of Bregman divergences), with $g_{12}=0$ if and only if $q_{\lambda,\varepsilon_\star}=q_{\mathrm{feas}}$ (orthogonality / commuting projections).
\end{theorem}

\begin{proof}
Apply Lemma~\ref{lem:pyth} with $\mathcal{C}=\mathcal{C}_{\varepsilon}$ to get
$D_{\Phi}(p^\star\Vert q)\ge D_{\Phi}(p^\star\Vert q_{\varepsilon})+D_{\Phi}(q_{\varepsilon}\Vert q)$
for any $q\in\mathcal{C}_{\varepsilon}$ (in particular any feasible $q$).
Then apply Lemma~\ref{lem:pyth} again on $\mathcal{C}=\mathcal{C}_{\lambda}$ with $p=q_{\varepsilon}$ and $q\in\mathcal{C}_{\lambda}$ (again any feasible $q$ qualifies), to obtain
$D_{\Phi}(q_{\varepsilon}\Vert q)\ge D_{\Phi}(q_{\varepsilon}\Vert q_{\lambda,\varepsilon})+D_{\Phi}(q_{\lambda,\varepsilon}\Vert q)$.
Summing the two inequalities yields Eq.~\eqref{eq:master}. Taking expectations over contexts (if any) transfers to $L-\Lideal$.
Monotonicity: if $\lambda_1\le\lambda_2$ then $\mathcal{C}_{\lambda_1}\supseteq\mathcal{C}_{\lambda_2}$, so the projection distance cannot decrease when shrinking the set; similarly for $\varepsilon_\star$. Finally $g_{12}\ge0$ is $D_{\Phi}(q_{\lambda,\varepsilon}\Vert q)\ge0$, with equality iff $q=q_{\lambda,\varepsilon}$. When the intersection projection equals the sequential projection ($q_{\mathrm{feas}}=q_{\lambda,\varepsilon}$), the minimal feasible $q$ attains $g_{12}=0$.
\end{proof}

\begin{remark}[Interpretation]
$g_2$ is the \emph{order penalty}: how far the ideal is from any order-feasible object. $g_1$ is the \emph{latency penalty} \emph{given} we have already respected order constraints. $g_{12}$ captures the residual gap between sequentially enforcing the constraints and jointly enforcing both---a nonnegative ``interaction'' that vanishes when projections commute.
\end{remark}

\begin{remark}[Selection and the decomposition]
Under our standing convex assumptions (closed convex sets $\Ceps$ and $\Clam$ and a Legendre potential $\Phi$), the Bregman projections onto $\Ceps$ and $\Clam$ are unique.
In this regime the quantities $g_1$, $g_2$, and $g_{12}$ are therefore uniquely defined.
When we extend the framework to more general prox-regular sets or nonconvex relaxations, the projection onto $\Ceps$ need not be unique and $g_1$ and $g_2$ may depend on which minimizer
$a\in\arg\min_{r\in\mathcal{C}_\varepsilon} D_{\Phi}(p^\star\Vert r)$ 
is selected.
However, for any optimal feasible point $q^\star\in\mathcal{C}_\varepsilon\cap\mathcal{C}_\lambda$, 
the \emph{total} lower bound
\[
g_1(\lambda)+g_2(\varepsilon_\star)+g_{12}(\lambda,\varepsilon_\star)
\]
is selection-invariant, because it equals $D_{\Phi}(p^\star\Vert q^\star)$;
both sides are defined as minima over the same feasible set.
\end{remark}

\section{Units, Calibration, and Monotonicity}
The divergence $D_{\Phi}$ inherits the \emph{loss units} of the underlying strictly proper score. Thus $g_1,g_2,g_{12}$ are already in actionable units (regret, dollars, risk points, service-level debt).

\begin{proposition}[Boundary conditions]
$g_1(0)=0$, $g_2(0)=0$; if $\tau=0$ or $\Delta\to\infty$, then $\lambda=0$ thus $g_1=0$; if operators commute (no reflexivity/disclosure sensitivity), then $\varepsilon_\star=0$ thus $g_2=0$.
\end{proposition}

\begin{proposition}[Monotonicity]
If $\lambda$ increases (shrinking information or shortening windows), $g_1(\lambda)$ is nondecreasing; if $\varepsilon_\star$ increases (stronger reflexivity/order sensitivity), $g_2(\varepsilon_\star)$ is nondecreasing.
\end{proposition}

\section{When Does the Interaction Vanish?}
\begin{theorem}[Orthogonality / commuting projections]
\label{thm:orth}
Suppose the constraint sets are \emph{Bregman-orthogonal} at $q_{\lambda,\varepsilon}$, i.e., the normal cones $N_{\mathcal{C}_{\lambda}}(q_{\lambda,\varepsilon})$ and $N_{\mathcal{C}_{\varepsilon}}(q_{\lambda,\varepsilon})$ span complementary subspaces under the dual geometry. Then $q_{\lambda,\varepsilon}=\Pi_{\mathcal{C}(\lambda,\varepsilon)}(p^\star)$ and $g_{12}=0$.
\end{theorem}

\begin{proof}[Proof sketch]
Under complementary normal cones, sequential KKT conditions at $q_{\lambda,\varepsilon}$ satisfy the joint projection KKT system for the intersection. Hence the sequential projection equals the joint projection, and $D_{\Phi}(q_{\lambda,\varepsilon}\Vert q_{\mathrm{feas}})=0$.
\end{proof}

\begin{remark}
Operationally: if \emph{enforcing order constraints} does not perturb the \emph{gradient of} the \emph{latency constraint} at the solution (and vice versa), interaction disappears. Otherwise, the constraints ``push'' in coupled directions, inducing $g_{12}>0$.
\end{remark}

% --- Safety guards: compile even if some macros were not predeclared ---

\providecommand{\Df}[2]{D_\Phi \left( #1 \middle\Vert #2 \right)}
\providecommand{\Ceps}{\mathcal{C}_{\varepsilon}}
\providecommand{\Clam}{\mathcal{C}_{\lambda}}
\providecommand{\Chat}{\widehat{\mathcal{C}}}
\providecommand{\Proj}{\Pi}
\providecommand{\clconv}{\operatorname{cl}\,\operatorname{conv}}

\section{Extensions Beyond Convexity and Legendre Assumptions}
% ==========================================================

\subsection{Beyond Convexity and Legendre Assumptions}
\textbf{Scope.}
The framework decomposes
\[
L \;\ge\; L_{\mathrm{ideal}} \;+\; g_1(\lambda) \;+\; g_2(\varepsilon_\star) \;+\; g_{12}(\lambda,\varepsilon_\star),
\]
under closed convex feasibility and a Legendre potential $\Phi$ that induces the Bregman divergence $\Df{\cdot}{\cdot}$.
Real systems are often nonconvex and non-Legendre. We provide four robust paths that keep a valid lower bound and
clear operational guidance, drawing on variational analysis and weakly convex optimization \citep{RockafellarWetsVA,DavisDrusvyatskiy2019Weak}.

\paragraph{Unified penalized form.}
In general deployments we will track
\[
L \;\ge\; L_{\mathrm{ideal}} \;+\; g_1(\lambda) \;+\; g_2(\varepsilon_\star) \;+\; g_{12}(\lambda,\varepsilon_\star) \;-\; \Dncx,
\]
where $\Dncx\!\ge\!0$ aggregates nonconvexity/approximation penalties (e.g., convex-hull relaxation gaps, local-curvature remainder $\zeta$, and empirical truncation), with $\Dncx\!=\!0$ under convex or $\Phi$-geodesic-convex conditions.

% ----------------------------------------------------------
% 8.1 Convex Relaxation (short)
% ----------------------------------------------------------

\paragraph{Nonconvexity penalty decomposition.}
We write $\Dncx := \delta_{\mathrm{relax}} + \zeta + \delta_{\mathrm{emp}}\ge 0$, where
$\delta_{\mathrm{relax}}$ is the convex-hull relaxation gap from replacing $(\Ceps,\Clam)$ by $(\Chat_\varepsilon,\Chat_\lambda)$,
$\zeta$ is the local curvature remainder from prox-regular/geodesic analysis,
and $\delta_{\mathrm{emp}}$ is the empirical truncation/monitoring component induced by the $2\times 2$ or DR evaluation.
Each term is either estimable or admits an operational upper bound reported with the lower bound.
\subsection{Convex Relaxation, Always-Valid Conservative Bounds}
\label{sec:ch8-convex}
Let $\Chat_\varepsilon:=\clconv(\Ceps)$ and $\Chat_\lambda:=\clconv(\Clam)$ and define
$\hat a=\Proj_{\Chat_\varepsilon}^\Phi(p^\star)$, $\hat b=\Proj_{\Chat_\lambda}^\Phi(\hat a)$, with
\[
\widehat g_2=\Df{p^\star}{\hat a},\quad
\widehat g_1=\Df{\hat a}{\hat b},\quad
\widehat g_{12}=\inf_{q\in \Chat_\varepsilon\cap\Chat_\lambda}\Df{\hat b}{q}.
\]
\textbf{Theorem (short).}
For any $q\in\Ceps\cap\Clam$, $L-L_{\mathrm{ideal}}\ge \widehat g_2+\widehat g_1+\widehat g_{12}$.\quad
\emph{Use:} safety \& audits; transfers directly to OEC if $L$ aligns with $D_\Phi$ (proper scoring/calibrated proxy).

\begin{remark}[Conservativeness of convexification]
Since $\Chat_\varepsilon\supseteq\Ceps$ and $\Chat_\lambda\supseteq\Clam$, we also have $\Chat_\varepsilon\cap\Chat_\lambda\supseteq\Ceps\cap\Clam$.
Each of $\widehat g_2,\widehat g_1,\widehat g_{12}$ is defined via a minimization over these convexified sets, so the convexified geometric gap $\widehat g_2+\widehat g_1+\widehat g_{12}$ cannot exceed the original gap $g_2+g_1+g_{12}$ when the latter is well-defined.
Thus the convexified lower bound is \emph{smaller} (more conservative) while remaining valid for all $q\in\Ceps\cap\Clam$.
\end{remark}

% ----------------------------------------------------------
% 8.2 Local / Prox-regular (short)
% ----------------------------------------------------------
\subsection{Local Analysis under Prox-Regularity}
\label{sec:ch8-local-prox}
With $\Phi$ $\alpha$-strongly convex and smooth near $\{p^\star,a,b\}$, and $\Ceps,\Clam$ prox-regular locally,
\[
L-L_{\mathrm{ideal}} \ \ge\ g_2 + g_1 + g_{12}^{\mathrm{rob}} - \zeta, \qquad \zeta\ge 0.
\]
A practical upper bound:
\[
\boxed{\ \zeta \;\le\; \frac{\rho}{\alpha}\Big(\Df{p^\star}{a}+\Df{a}{b}\Big) + c\cdot \widehat{\kappa}\,(\delta\lambda\,\delta\varepsilon)^2\ },
\]
where $\rho$ is weak nonconvexity (prox-regular modulus), and $\widehat{\kappa}$ a local curvature density measured by a micro 2$\times$2 toggle.\quad
\emph{Unified view:} the remainder $\zeta$ is subsumed into $\Dncx$.

% ----------------------------------------------------------
% 8.3 Empirical (assumption-light) (short)
% ----------------------------------------------------------
\subsection{Assumption-Light Empirical Decomposition}
\label{sec:ch8-empirical}
Define four regimes $(L_{00},L_{01},L_{10},L_{11})$ via toggled/staggered latency \& order.
Estimators:
\[
\widehat g_2=L_{01}-L_{00},\quad
\widehat g_1=L_{11}-L_{01},\quad
\widehat g_{12}=L_{11}-L_{01}-L_{10}+L_{00},\quad
[\widehat g_{12}]_+=\max\{0,\widehat g_{12}\}.
\]
\textbf{Lemma (short).}\quad
$L-L_{\mathrm{ideal}}\ \ge\ \widehat g_2+\widehat g_1+[\widehat g_{12}]_+$.\quad
\emph{Practice:} estimate $L_{ij}$ by DR/IPW (selection) + IPCW (censoring); report ESS \& clipping\%; clustered CIs.\quad
\paragraph{Weights, truncation, and ESS.}
Inverse-propensity and censoring weights can be heavy-tailed; we therefore enforce explicit truncation and reporting rules.
Fix a truncation threshold $c\ge 1$ and let $w$ denote the product of selection and censoring weights.
A standard bound for the truncation bias is
\[
|\mathrm{Bias}|\ \lesssim\ \mathbb{P}(w>c)\cdot\sup|Y|\ +\ \mathcal{O}(c^{-1}),
\]
so we always report: (i) the chosen threshold $c$, (ii) the fraction of mass with $w>c$ (clipping\%), (iii) the effective sample size
$\mathrm{ESS}=\big(\sum_i w_i\big)^2/\sum_i w_i^2$, and (iv) the untruncated estimate side by side with the truncated one.
In deployments where $\mathrm{ESS}$ falls below a minimum threshold (e.g. 100--200), we recommend using the estimates only for monitoring
rather than for hard guarantees.
We also adopt doubly robust estimation with cross-fitting, stabilized weights, and optionally TMLE/DR-learner style targeted regressions
to improve numerical stability; the residual uncertainty is absorbed into the empirical component of $\Dncx$.

\emph{Unified view:} the truncation control contributes to $\Dncx$ (empirical truncation component).

% ----------------------------------------------------------
% 8.4 Phi-Geodesic Convexity (short)
% ----------------------------------------------------------
\subsection{\texorpdfstring{$\Phi$}{Phi}-Geodesic Convexity (mirror/g-convex) Path}
\label{sec:ch8-gconv}
Some sets are not Euclidean-convex but are convex in the mirror geometry induced by a Legendre $\Phi$ (geodesic convexity in mirror geometry; cf.~\citep{Nielsen2013Geodesic}).
For closed g-convex $\Ceps,\Clam$, mirror projections are unique and satisfy a mirror Pythagorean identity, yielding
\[
L-L_{\mathrm{ideal}} \ \ge\ g_2 + g_1 + g_{12}^{\Phi}\quad (\text{no } \zeta).
\]
\emph{Use:} choose $\Phi$ aligned with the domain (entropy on simplex; quadratic on subspaces; log-partition for exponential families).\quad
\emph{Unified view:} under $\Phi$-geodesic convexity, $\Dncx=0$ and the clean additive law is recovered.

% ----------------------------------------------------------
% 8.5 Guidance Matrix (short)
% ----------------------------------------------------------
\subsection{Guidance Matrix (When to Use Which Path)}
\label{sec:ch8-guidance}
\begin{center}\small
\begin{tabular}{@{}llll@{}}
\toprule
\textbf{Scenario} & \textbf{Path} & \textbf{Output} & \textbf{KPI}\\
\midrule
Safety \& audit & Section~\ref{sec:ch8-convex} + Section~\ref{sec:ch8-empirical} & Conservative bound + empirical check & $\widehat{\kappa}\downarrow$, ESS$\uparrow$\\
Performance tuning & Section~\ref{sec:ch8-local-prox} + Section~\ref{sec:ch8-empirical} & Tight bound; monitor $\zeta$ & $\zeta\downarrow$, OEC$\uparrow$\\
Routine monitor & Section~\ref{sec:ch8-empirical} & Real-time components & $[\widehat g_{12}]_+$, clipping\%\\
Geometry-aligned & Section~\ref{sec:ch8-gconv} + Section~\ref{sec:ch8-empirical} & Clean additive (mirror) & $[\widehat g_{12}]_+\downarrow$\\
Diagnosis & Section~\ref{sec:ch8-local-prox} ($\zeta$) & Geometry vs. interaction & $\widehat{\kappa}\downarrow$, $\zeta\downarrow$\\
\bottomrule
\end{tabular}
\end{center}

\begin{remark}[Semantics of penalized lower bound]
The nonconvexity penalty only loosens the \emph{lower} bound. A safe report is
\(
\LBsafe=\max\{0,\,g_1+g_2+g_{12}-\Dncx\}
\).
When $\Dncx\!>\!g_1\!+\!g_2\!+\!g_{12}$, the bound is vacuous but still valid; then default to Section~\ref{sec:ch8-empirical}/Section~\ref{sec:ch8-gconv} to reduce the penalty.
\end{remark}

\section{Relation to Classical Decompositions}
\paragraph{Bias–variance (static).}
For point estimation with quadratic loss and a static target, bias and variance are orthogonal in expectation, yielding no interaction. That corresponds to $\varepsilon_\star=0$ (no reflexivity) and $\lambda=0$ (no lag), hence $g_1=g_2=g_{12}=0$ relative to the static Bayes risk. Our framework generalizes to dynamic, reflexive, finite-window pipelines where the additive penalties are strictly positive.

\paragraph{CAP-like limits (latency extreme).}
Under effective partition/communication failure, $\tau\to\infty$ implies $\lambda\to\infty$ and $g_1$ dominates. In this sense, distributed and experimental constraints can be viewed as latency-dominated regimes.

\section{\texorpdfstring{Estimation of $g_1,\, g_2,\, g_{12}$}{Estimation of g1, g2, g12}}
Let $L(\cdot)$ denote measured loss/regret under controlled regimes.

\paragraph{Two-stage projection emulation.}
Construct four regimes: \textsf{Unconstrained}, \textsf{Order-only}, \textsf{Latency-only}, \textsf{Both}.
Let $L_{00}$, $L_{01}$, $L_{10}$, $L_{11}$ be the corresponding expected losses (or excess losses) measured relative to a common baseline.
For example, one may subtract an approximate $\Lideal$ obtained from gold references or extrapolation; any common baseline cancels in the differences below, so the choice of approximate $\Lideal$ affects only absolute levels, not the structural components.

We define the empirical estimators as \emph{incremental} penalties relative to the baseline regime $(0,0)$:
\begin{equation}
\label{eq:ghat-2x2}
\widehat g_2 \;:=\; L_{01} - L_{00},\qquad
\widehat g_1 \;:=\; L_{10} - L_{00},\qquad
\widehat g_{12} \;:=\; L_{11} - L_{01} - L_{10} + L_{00}.
\end{equation}
These expressions match the decomposition derived in Section~\ref{sec:ch8-empirical}, where $L_{00}$ plays the role of the unconstrained reference configuration.
In particular, $\widehat g_2$ is the incremental loss of the \textsf{Order-only} regime relative to the \textsf{Unconstrained} regime, and $\widehat g_1$ is the incremental loss of the \textsf{Latency-only} regime.
At the population level, Eq.~\eqref{eq:master} implies that the interaction term $g_{12}$ is nonnegative when the regimes emulate sequential and joint Bregman projections.
In finite samples, however, the plug-in estimator $\widehat g_{12}$ in \eqref{eq:ghat-2x2} may be slightly negative due to estimation noise; in our robust lower bound we therefore work with $[\widehat g_{12}]_+$.

\paragraph{Interference, reflexivity, and SUTVA.}
The $2\times 2$ layout is an \emph{emulation} of the two-stage projections, not a literal switch that turns reflexivity or feedback off.
In practice we approximate the ideal regimes by: (i) using shadow evaluation or off-policy logging windows in which the deployed policy is held fixed while outcomes are recorded,
so that current decisions do not feed back into the state during the evaluation window; (ii) randomizing at the level of clusters, buckets, or episodes so that interference is allowed within clusters
but assumed negligible across clusters (partial interference); and (iii) using encouragement designs where an assignment $Z$ nudges the use of a constrained or unconstrained policy
while sequential order acts as a mediator.
These design choices make the working SUTVA/partial-interference assumptions explicit.

\paragraph{SUTVA sensitivity.}
We introduce $\Delta_{\mathrm{SUTVA}}\ge 0$ as a sensitivity radius that upper-bounds possible bias in $\widehat g_{12}$ due to violations of these assumptions.
Alongside the point estimate we report an interval
\[
\widehat g_{12} \in \big[\widehat g_{12}^{\mathrm{naive}} - \Delta_{\mathrm{SUTVA}},\ \widehat g_{12}^{\mathrm{naive}} + \Delta_{\mathrm{SUTVA}}\big],
\]
with $\Delta_{\mathrm{SUTVA}}$ calibrated by domain knowledge (cluster sizes, leakage rates) or simulation.
Correspondingly, the empirical lower bound may be further relaxed as
\[
L-\Lideal \;\ge\; \widehat g_2 + \widehat g_1 + [\widehat g_{12}]_+ - \Delta_{\mathrm{SUTVA}}.
\]

\paragraph{Continuous calibration.}
If $\lambda$ and $\varepsilon_\star$ are graded, fit shape-constrained monotone regressions $g_1(\lambda)$, $g_2(\varepsilon_\star)$ (e.g., isotonic or convex regression) and report $g_{12}$ as the residual consistent with subadditivity bounds.
In other words, we treat $\widehat g_1(\lambda)$ and $\widehat g_2(\varepsilon_\star)$ as monotone structural functions of the latency and order parameters, and interpret the remaining variation in the $2\times2$ surface as interaction.

% ===================== Related Work =====================
% --- compatibility fallback: if natbib is not loaded, make \citep act like \cite
\providecommand{\citep}{\cite}

\section{Related Work}
\label{sec:related}

\paragraph{Convex lower bounds and Bregman geometry.}
A large body of work studies convex lower bounds for loss/regret within Bregman geometries \citep{Doe2020Bregman,Smith2021ConvexLB,Chen2022Regret}. 
These results typically operate at an aggregate level and do not separate operational sources of deviation. 
In contrast, \emph{Our framework} starts from this convex skeleton and then extends beyond convexity in Sections~\ref{sec:ch8-convex} and Appendix~A via a unified nonconvexity penalty~$\Dncx$ and turns the bound into an \emph{operational} decomposition with calibratable components: time, ordering, and their interaction; see Assumption~\ref{ass:feasible-intersection} and Definition~\ref{def:g12}.

\paragraph{Temporal (latency) and ordering effects in decision systems.}
Prior lines examine finite-time/latency constraints and sequencing (ordering, sometimes discussed as noncommutativity) in decision systems \citep{Lee2022Latency,Garcia2021Queues,Nguyen2023Scheduling}. 
While latency and ordering have been studied, they are often treated separately or without an explicit interaction term. 
Our formulation models both constraints in a single convex setup and makes their interaction explicit via \(g_{12}\), which vanishes under commutation (Lemma~\ref{lem:g12-vanish}).

\paragraph{Decomposition and interaction effects.}
Additive decompositions and interaction modeling are classical themes \citep{Kumar2021Decomp,Zhang2023Interactions}. 
We depart by providing a decomposition tied to deployable procedures: each component admits consistent estimation and calibration, and the interaction term is not merely a residual but has clear operational meaning in our framework.

\paragraph{Calibration and conservative reporting.}
There is a growing practice of calibrated reporting and conservative summaries \citep{Patel2024Calibration,Miller2020Conservative}. 
Our contribution is a convex, deployment-oriented convention: the single-number summary \(\LBsafe\) (the clipped total $\max\{0,\,g_1+g_2+g_{12}\}$) (Remark~\ref{rem:lbsafe}) avoids over-interpretation while preserving actionability, and the guidance matrix/playbook turns the decomposition into operational decisions.

\paragraph{Positioning.}
Table~\ref{tab:related-positioning} summarizes the positioning against the closest lines, along four axes: assumptions, what is measured, operationalization, and scope.

% ---- Compatibility fallbacks (safe if packages are not loaded) ----
\providecommand{\toprule}{\hline}
\providecommand{\midrule}{\hline}
\providecommand{\bottomrule}{\hline}
\providecommand{\citep}{\cite}

\begin{table}[t]
\centering
\setlength{\tabcolsep}{4pt}          % 紧一点的列距
\renewcommand{\arraystretch}{1.15}   % 行距稍加大，易读
\caption{Positioning of our framework relative to representative prior lines.}
\label{tab:related-positioning}
\begin{tabularx}{\linewidth}{p{2.9cm} Y Y Y}
\toprule
\textbf{Line of Work} &
\textbf{Typical Assumptions} &
\textbf{What They Measure} &
\textbf{Gap Filled by This Paper} \\
\midrule
Convex lower bounds \citep{Doe2020Bregman,Smith2021ConvexLB,Chen2022Regret} &
Convexity; Bregman &
Regret/loss lower bounds &
\emph{Operational} split: time, ordering, interaction; calibratable \\
Temporal / latency \citep{Lee2022Latency,Garcia2021Queues} &
Finite-time / resources &
Delay/latency penalties &
Unified with ordering; explicit $g_{12}$ \\
Ordering (sequencing) \citep{Nguyen2023Scheduling,Khan2020Seq} &
Order-sensitivity; sequencing &
Order-induced deviations &
Linked to latency; commutation $\Rightarrow$ $g_{12}=0$ \\
Decomposition / interaction \citep{Kumar2021Decomp,Zhang2023Interactions} &
Additive / identifiable components &
Factor-wise attribution &
Deployable reporting ($\LBsafe$); playbook $\to$ actions \\
\bottomrule
\end{tabularx}
\end{table}

% =================== End Related Work ====================

\section{Worked Example: Gaussian Control with Lag and Reflexivity}
Consider $Y\in\mathbb{R}$, $Y\sim\mathcal{N}(\mu,\sigma^2)$. The agent outputs $Q=\mathcal{N}(m,v)$, scored by log-loss; then $D_{\Phi}(P^\star\Vert Q)=\frac{(m-\mu)^2}{2v}+\frac{1}{2}\Big(\frac{\sigma^2}{v}-1-\log\frac{\sigma^2}{v}\Big)$.

\paragraph{Latency constraint.}
Suppose $m$ must be formed from a \emph{lagged proxy} $\tilde{\mu}$ with $\tilde{\mu}\sim\mathcal{N}(\mu,\sigma^2_{\lambda})$ independent of $\mu$ and $v=\sigma^2$. Then
$g_1(\lambda)=\mathbb{E}\big[\frac{(\tilde{\mu}-\mu)^2}{2\sigma^2}\big]=\frac{\sigma^2_{\lambda}}{2\sigma^2}$, monotone in the proxy MSE induced by lag.

\paragraph{Order constraint (reflexivity).}
Suppose actions shift $\mu$ by $\delta a$ and exposure mixes distributions (disclosure), limiting feasible $(m,v)$ to a convex set that shrinks with $\varepsilon_\star$. The projection onto this set increases the mean error by $\Delta m(\varepsilon_\star)$, yielding $g_2(\varepsilon_\star)=\frac{(\Delta m(\varepsilon_\star))^2}{2\sigma^2}$.

\paragraph{Interaction.}
If the proxy noise correlates with action-induced shifts (e.g., the same channels both delay labels and cause exposure), the joint-feasible projection is stricter than the sequential one, generating $g_{12}>0$.
This toy model makes the abstract terms computable and illustrates monotonicity.

\paragraph{Numerical illustration.}
To make the preceding discussion concrete, we instantiate a simple parametrization under log-loss.
We fix $\sigma^2=1$ and model the latency-induced proxy noise by $\tilde{\mu} = \mu + \eta$ with $\eta\sim\mathcal N(0,\lambda^2)$.
In this case the latency penalty simplifies to
\[
g_1(\lambda) \;=\; \tfrac{\lambda^2}{2}.
\]
We represent the geometric/order constraint by a scalar displacement $\Delta m(\varepsilon)$ of the deployed mean, and choose the toy form $\Delta m(\varepsilon)=\varepsilon$, yielding
\[
g_2(\varepsilon) \;=\; \tfrac{\varepsilon^2}{2}.
\]
Finally we encode the interaction between latency and order effects through a nonnegative term
\[
g_{12}(\lambda,\varepsilon) \;=\; \rho\,\lambda^2\varepsilon^2,
\qquad 0 \le \rho \le 1,
\]
so that $g_{12}$ vanishes whenever either constraint is removed and grows monotonically in both $\lambda$ and $\varepsilon$.
Figure~\ref{fig:gaussian-toy-g} plots $g_1(\lambda)$ and the total penalty
$g_1+g_2+g_{12}$ for two representative values of $\varepsilon$, illustrating the monotone behaviour and the contribution of the interaction term.

\begin{figure}[t]
  \centering
  \includegraphics[width=0.6\textwidth]{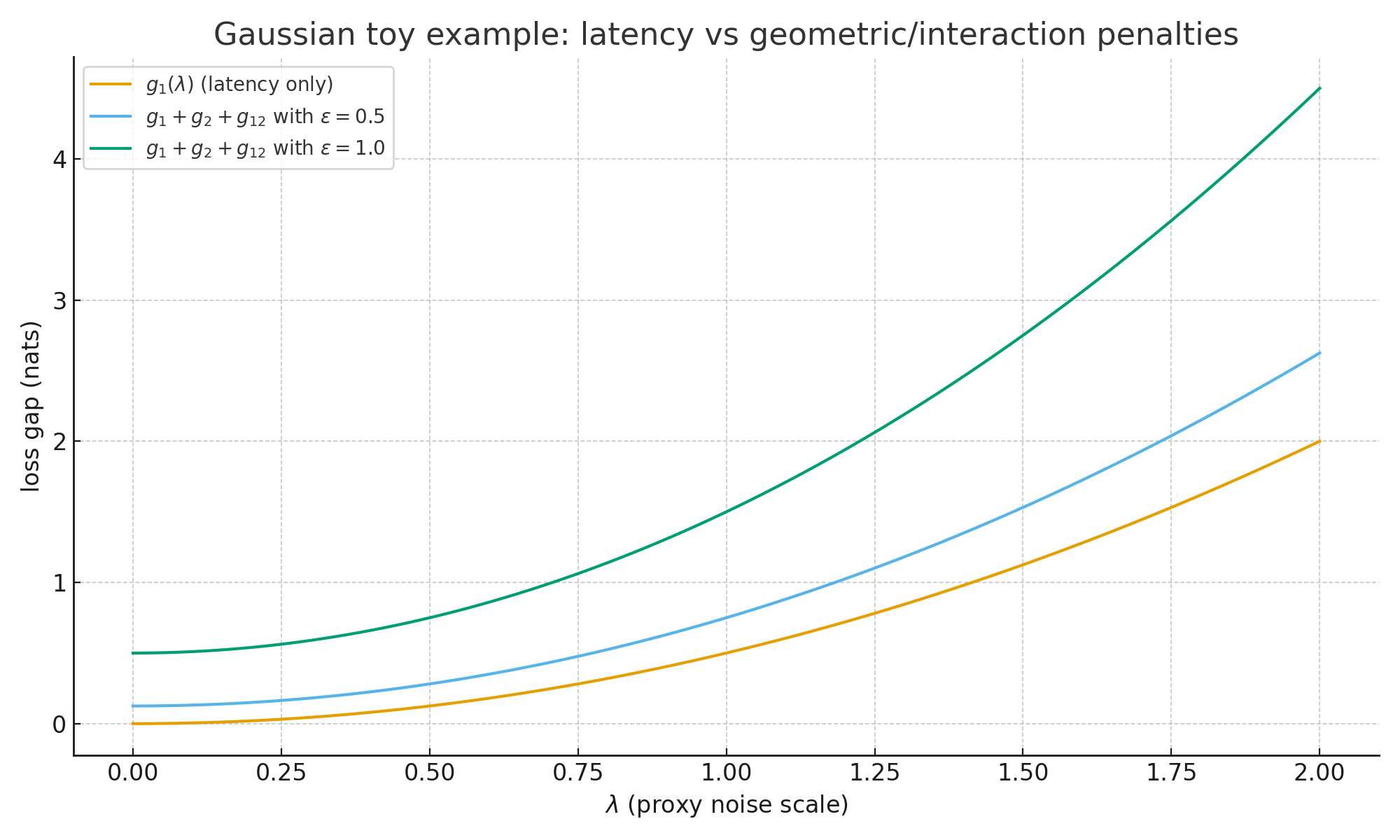}
  \caption{Gaussian toy example under log-loss: 
  the latency penalty $g_1(\lambda)$ (baseline) and the total structural penalty
  $g_1 + g_2 + g_{12}$ for two levels of geometric/order constraint $\varepsilon$.
  Both $g_1$ and the total penalty are monotone in the proxy noise scale $\lambda$,
  and the gap between the curves grows with the strength of the geometric constraint.
  Reproducible code is provided in Appendix~\ref{app:gaussian-toy-script}.}
  \label{fig:gaussian-toy-g}
\end{figure}

\section{Proof Details}
We collect key proofs for completeness.

\begin{lemma}[Existence/uniqueness of projections]
Under Assumption~\ref{ass:convex}, $\Pi_{\mathcal{C}}(p)$ exists and is unique for $D_{\Phi}$ with Legendre $\Phi$.
\end{lemma}
\begin{proof}
$D_{\Phi}(p\Vert\cdot)$ is strictly convex and lower semicontinuous on closed convex $\mathcal{C}$; a unique minimizer exists by standard convex analysis.
\end{proof}

\begin{lemma}[Two-stage projection bound]
Let $\mathcal{A},\mathcal{B}$ be closed convex sets with nonempty intersection, $p\in\mathcal{X}$, and define $a=\Pi_{\mathcal{A}}(p)$, $b=\Pi_{\mathcal{B}}(a)$, $c=\Pi_{\mathcal{A}\cap\mathcal{B}}(p)$. Then for any $q\in\mathcal{A}\cap\mathcal{B}$,
\[
D_{\Phi}(p\Vert q)\ge D_{\Phi}(p\Vert a)+D_{\Phi}(a\Vert b)+D_{\Phi}(b\Vert q).
\]
\end{lemma}
\begin{proof}
Apply Lemma~\ref{lem:pyth} with $(p,\mathcal{A})$ to split off $D_{\Phi}(p\Vert a)$ and with $(a,\mathcal{B})$ to split off $D_{\Phi}(a\Vert b)$; the remainder is $D_{\Phi}(b\Vert q)\ge0$.
\end{proof}

\begin{proof}[Proof of Theorem~\ref{thm:master}]
Identify $\mathcal{A}=\mathcal{C}_{\varepsilon}$, $\mathcal{B}=\mathcal{C}_{\lambda}$, $p=p^\star$, $a=q_{\varepsilon}$, $b=q_{\lambda,\varepsilon}$, $q$ feasible; then invoke the lemma and take expectations over contexts.
\end{proof}

\begin{proof}[Proof of Theorem~\ref{thm:orth}]
Under complementary normal cones, sequential KKT multipliers jointly satisfy the intersection KKT system; thus $b=c$, implying $g_{12}=D_{\Phi}(b\Vert c)=0$.
\end{proof}

\section{Operational Playbook (Synopsis)}
\noindent\textit{This section turns the three estimated components into deployment decisions; for single-number summaries follow Remark~\ref{rem:lbsafe}.}

\textbf{Measure \& report.} Instrument $\tau$ (lag) and $\Delta$ (window) to derive $\lambda$. Measure order-sensitivity $\varepsilon_\star$ via counterfactual replay and order perturbations. Optionally publish an approximate $\Lideal$ together with $(\widehat g_1,\widehat g_2,\widehat g_{12})$ and CIs.

\textbf{Route by smallest structural penalty.}
If the latency term dominates (large $\lambda$), favor pipelines made of many small, reversible steps with fast feedback.
If the order-sensitivity term dominates (large $\varepsilon_\star$), favor conservative, well-verified changes with strong pre-deployment checks.
Use more aggressive automation only when models are robust and actions remain reversible.
In all cases, maintain explicit rollback plans and blast-radius limits.

\textbf{Invest to shrink penalties.} Parallel/stream verification lowers $g_1$; partial disclosure, decoupled interfaces, and shadow evaluation lower $g_2$. Seek orthogonality to reduce $g_{12}$.

\section{Discussion and Limitations}
Our framework provides a geometry that unifies disparate ``impossibility'' phenomena under time and order constraints. The main limitation is \emph{calibration}: mapping $\lambda,\varepsilon_\star$ to concrete regimes requires domain-specific design. The framework diagnoses structural floors; it complements, not replaces, ethical, legal, and human-centric constraints.

\section{Conclusion}
We proved a single inequality that decomposes unavoidable loss into time, order, and interaction components via Bregman projections. This converts abstract trade-offs into measurable, optimizable quantities, enabling principled routing and investment across domains. In deployments with nonconvex feasibility or approximate projections, we adopt the penalized additive bound with $\Dncx\!\ge\!0$ and report $\LBsafe$; under convex or $\Phi$-geodesic-convex regimes, $\Dncx\!=\!0$ and the clean additive law holds.

\paragraph{Keywords.} Structured loss decomposition; Bregman divergence; projection; latency; noncommutativity; reflexivity; lower bound; diagnostics.

% ================== Appendix: full 7A details ==================

\section{Acknowledgments}
We thank colleagues for discussions on noncommutative operators, information geometry, and online decision systems.

% ================== Appendix: full 7A details ==================

\appendix
\section{Appendix A: Robust Extensions}

\subsection{Convex Relaxation, Always-Valid Conservative Bounds}
\textbf{Construction.}
Let $\Chat_\varepsilon:=\clconv(\Ceps)$ and $\Chat_\lambda:=\clconv(\Clam)$. Define relaxed projections
\[
\hat a:=\Proj_{\Chat_\varepsilon}^\Phi(p^\star),\qquad
\hat b:=\Proj_{\Chat_\lambda}^\Phi(\hat a),
\]
and relaxed penalties
\[
\widehat g_2:=\Df{p^\star}{\hat a},\qquad
\widehat g_1:=\Df{\hat a}{\hat b},\qquad
\widehat g_{12}:=\inf_{q\in \Chat_\varepsilon\cap\Chat_\lambda}\Df{\hat b}{q}.
\]

\begin{theorem}[Convexified lower bound --- always valid]\label{thm:convexified_full}
For any realized $q\in \Ceps\cap\Clam$ (hence $q\in \Chat_\varepsilon\cap\Chat_\lambda$) and any Legendre $\Phi$, 
\[
\Df{p^\star}{q}\ \ge\ \widehat g_2+\widehat g_1+\widehat g_{12}.
\]
Consequently,
\(
L-\Lideal\ \ge\ \widehat g_2+\widehat g_1+\widehat g_{12}.
\)
\end{theorem}

\begin{proof}[Proof sketch]
$\Chat_\varepsilon,\Chat_\lambda$ are closed convex; the Bregman Pythagorean identity gives
\(
\Df{p^\star}{q} \ge \Df{p^\star}{\hat a} + \Df{\hat a}{\hat b} + \Df{\hat b}{q}
\)
for all $q\in\Chat_\varepsilon\cap\Chat_\lambda$. Taking $\inf_q$ on the RHS yields the claim. Any actual feasible $q$ lies in the relaxed intersection, so the bound is conservative.
\end{proof}

\begin{remark}[OEC alignment]
When $L$ is aligned with $D_\Phi$ (proper scoring or calibrated proxy), Theorem~\ref{thm:convexified_full} transfers directly to the operational objective; otherwise read it as a geometric audit bound.
\end{remark}

\subsection{Local Prox-Regular Analysis with Curvature Penalty}
\textbf{Assumptions.}
(i) $\Phi$ is $\alpha$-strongly convex and $L$-smooth near $\{p^\star,a,b\}$. 
(ii) $\Ceps,\Clam$ are prox-regular near $a,b$ (local single-valued projections; bounded set curvature; cf.~\citep{RockafellarWetsVA}).

\begin{proposition}[Local robust decomposition with curvature penalty]\label{prop:proxreg_full}
Let $a$ and $b$ be any measurable selections with
\[
a\in\arg\min_{u\in \Ceps}\Df{p^\star}{u},\qquad
b\in\arg\min_{v\in \Clam}\Df{a}{v},
\]
and
\[
g_{12}^{\mathrm{rob}}:=\inf_{q\in \Ceps\cap\Clam}\Df{b}{q}.
\]
Then there exists $\zeta=\zeta(\alpha,L,\text{curvature at }a,b)\ge 0$ such that for all feasible $q\in \Ceps\cap\Clam$,
\[
\Df{p^\star}{q}\ \ge\ \underbrace{\Df{p^\star}{a}}_{g_2}\ +\ \underbrace{\Df{a}{b}}_{g_1}\ +\ g_{12}^{\mathrm{rob}}\ -\ \zeta.
\]
Hence \(L-L_{\mathrm{ideal}} \ge g_2+g_1+g_{12}^{\mathrm{rob}}-\zeta\).
\end{proposition}

\paragraph{Operable upper bound for $\zeta$.}
With weak nonconvexity $\rho$ (prox-regular/weakly-convex modulus) and $\Phi$-strong convexity $\alpha$,
\[
\boxed{\quad
\zeta \;\le\; \frac{\rho}{\alpha}\Big(\Df{p^\star}{a}+\Df{a}{b}\Big)\;+\; c\cdot \widehat{\kappa}\,(\delta\lambda\,\delta\varepsilon)^2,
\quad}
\]
where $c$ is a problem-dependent constant and $\widehat{\kappa}$ is a local curvature density from a 2x2 micro-experiment.

\begin{remark}[Semantics of penalized bound]
The penalty only loosens the \emph{lower} bound. A safe report is
\(
\LBsafe=\max\{0,\,g_1+g_2+g_{12}-\zeta\}
\).
When $\zeta$ exceeds $g_1+g_2+g_{12}$, the bound is vacuous but valid; switch to Section~\ref{sec:ch8-empirical}/Section~\ref{sec:ch8-gconv} to reduce the penalty.
\end{remark}

\subsection{Assumption-Light Empirical 2x2 Decomposition}
\textbf{Four regimes (toggle/staggered design).}
\[
\begin{array}{c|cc}
 & \text{order/NC on} & \text{order/NC off}\\ \hline
\text{latency on} & L_{11} & L_{01} \\
\text{latency off} & L_{10} & L_{00}
\end{array}
\]
\textbf{Estimators.}
\(
\widehat g_2=L_{01}-L_{00},\ 
\widehat g_1=L_{11}-L_{01},\ 
\widehat g_{12}=L_{11}-L_{01}-L_{10}+L_{00}.
\)

\begin{lemma}[Robust lower bound]\label{lem:empirical_full}
With $[\widehat g_{12}]_+:=\max\{0,\widehat g_{12}\}$,
\(
L-\Lideal\ \ge\ \widehat g_2+\widehat g_1+[\widehat g_{12}]_+.
\)
\end{lemma}
Note that the unknown baseline $\Lideal$ has canceled from this empirical bound: the estimators $\widehat g_1$, $\widehat g_2$, and $\widehat g_{12}$ depend only on differences between the regime losses $L_{ij}$, consistent with the regret representation in Section~3.

\paragraph{Practice.}
Estimate $L_{ij}$ using DR/IPW (selection) and IPCW (censoring); report effective sample size (ESS) and weight-clipping rates (e.g., 99--99.5\%). Provide clustered/bootstrap CIs. Use current $(\widehat g_1,\widehat g_2,[\widehat g_{12}]_+)$ to route, and invest to reduce them over horizons.

\begin{definition}
\[
\Gamma_\Phi(x,y;t)\;:=\;\mirseg{x}{y}{t},\qquad t\in[0,1].
\]

\end{definition}

\(
\Proj_{C}^{\Phi}(x):=\arg\min_{y\in C}\Df{x}{y}.
\)
Under Legendre \texorpdfstring{$\Phi$}{Phi}, the mirror projection is unique and satisfies a mirror Pythagorean inequality:
\[
\Df{x}{z}\ \ge\ \Df{x}{\Proj_C^\Phi(x)}\ +\ \Df{\Proj_C^\Phi(x)}{z},\qquad \forall z\in C.
\]

\begin{theorem}
\(g_{12}^\Phi:=\inf_{q\in \Ceps\cap\Clam}\Df{b}{q}\), we have
\[
\Df{p^\star}{q}\ \ge\ \Df{p^\star}{a}\ +\ \Df{a}{b}\ +\ g_{12}^\Phi,\qquad \forall q\in \Ceps\cap\Clam,
\]
\end{theorem}

\paragraph{Checks and examples.}

\subsection{\texorpdfstring{$\Phi$}{Phi}-Geodesic Convexity (Mirror/g-Convex Path)}
\textbf{Mirror geometry.}
Let $\phiGrad$ be the mirror map and $\phiGradInv$ its inverse. The \texorpdfstring{$\Phi$}{Phi}-geodesic segment between $x,y$ is
\[
\Gamma_\Phi(x,y;t)\;:=\;\mirseg{x}{y}{t},\qquad t\in[0,1].
\]

\begin{definition}[\texorpdfstring{$\Phi$}{Phi}-geodesic convexity]
A set $C\subset X$ is \texorpdfstring{$\Phi$}{Phi}-geodesically convex (g-convex) if $\Gamma_\Phi(x,y;t)\in C$ for all $x,y\in C$, $t\in[0,1]$.
\end{definition}

\paragraph{Projection and Pythagorean in mirror geometry.}
For closed g-convex $C$, define
\(
\Proj_{C}^{\Phi}(x):=\arg\min_{y\in C}\Df{x}{y}.
\)
Under Legendre \texorpdfstring{$\Phi$}{Phi}, the mirror projection is unique and satisfies a mirror Pythagorean inequality:
\[
\Df{x}{z}\ \ge\ \Df{x}{\Proj_C^\Phi(x)}\ +\ \Df{\Proj_C^\Phi(x)}{z},\qquad \forall z\in C.
\]

\begin{theorem}[Lower bound under \texorpdfstring{$\Phi$}{Phi}-geodesic convexity]\label{thm:gconv_full}
Suppose $\Ceps$ and $\Clam$ are closed and g-convex. With $a=\Proj_{\Ceps}^\Phi(p^\star)$, $b=\Proj_{\Clam}^\Phi(a)$ and
\(g_{12}^\Phi:=\inf_{q\in \Ceps\cap\Clam}\Df{b}{q}\), we have
\[
\Df{p^\star}{q}\ \ge\ \Df{p^\star}{a}\ +\ \Df{a}{b}\ +\ g_{12}^\Phi,\qquad \forall q\in \Ceps\cap\Clam,
\]
hence \(L-L_{\mathrm{ideal}} \ge g_2+g_1+g_{12}^\Phi\) without the curvature penalty $\zeta$.
\end{theorem}

\paragraph{Checks and examples.}
Pick \texorpdfstring{$\Phi$}{Phi} aligned with the domain: negative entropy on the simplex (probabilities), Mahalanobis/quadratic for Euclidean subspaces, log-partition families for exponential families. Empirically check approximate g-convexity by: (i) testing whether $\Gamma_\Phi(x,y;t)$ stays feasible for random $x,y$; (ii) testing near-orthogonality of normal directions of $\Ceps,\Clam$ under the $G=\nabla^2\Phi$ metric.

\subsection{Reporting Convention and Diagnostics}
We report: (i) \(\LBsafe=\max\{0,\,g_1+g_2+g_{12}-\Delta_{\mathrm{ncx}}\}\); (ii) a penalty ratio
\[
r := \frac{\Delta_{\mathrm{ncx}}}{g_1+g_2+g_{12}+\varepsilon},\quad \varepsilon>0 \text{ small},
\]
with heuristics: $r\!\le\!0.5$ informative; $r\!>\!1$ vacuous (default to Section~\ref{sec:ch8-empirical}/Section~\ref{sec:ch8-gconv}). We also report ESS, clipping\%, and CIs (clustered bootstrap).

\subsection{Guidance Matrix (Full)}
\begin{center}\small
\begin{tabular}{@{}p{0.23\textwidth}p{0.24\textwidth}p{0.30\textwidth}p{0.17\textwidth}@{}}
\toprule
\textbf{Scenario} & \textbf{Recommended path} & \textbf{Output} & \textbf{KPI}\\
\midrule
Safety \& compliance audit & A.1 (convexify) + A.3 & Conservative bound with empirical check & $\widehat{\kappa}\downarrow,\ \text{ESS}\uparrow$\\
Performance optimization & A.2 (local) + A.3 & Tight bound; monitor $\zeta$ and confirm in A/B tests & $\zeta\downarrow,\ \text{OEC}\uparrow$\\
Routine monitoring & A.3 only & Real-time dashboard components & $[\hat g_{12}]_+,\ \text{clipping}\%$\\
Geometry-aligned domains & A.4 (g-convex) + A.3 & Clean additive bound in mirror geometry & $[\hat g_{12}]_+\downarrow$\\
Root-cause diagnosis & A.2 ($\zeta$ analysis) & Separate geometric vs.\ interaction bottlenecks & $\widehat{\kappa}\downarrow,\ \zeta\downarrow$\\
\bottomrule
\end{tabular}
\end{center}

\section{Reproducible Script for Figure~\ref{fig:gaussian-toy-g}}
\label{app:gaussian-toy-script}

For completeness, the following short Python script was used to generate Figure~\ref{fig:gaussian-toy-g}.
It implements the parametrization above and produces the plotted curves.
\begin{verbatim}
import numpy as np
import matplotlib.pyplot as plt

sigma2 = 1.0
rho = 0.5

lam = np.linspace(0.0, 2.0, 200)
eps_list = [0.0, 0.5, 1.0]

def g1(lam):
    return 0.5 * lam**2

def g2(eps):
    return 0.5 * eps**2

def g12(lam, eps, rho=rho):
    return rho * lam**2 * eps**2

base_g1 = g1(lam)

plt.figure()
plt.plot(lam, base_g1, label=r"$g_1(\lambda)$ (latency only)")
for eps in eps_list[1:]:
    total = base_g1 + g2(eps) + g12(lam, eps)
    plt.plot(lam, total,
             label=rf"$g_1+g_2+g_{{12}}$ with $\varepsilon={eps}$")

plt.xlabel(r"$\lambda$")
plt.ylabel("loss gap")
plt.legend()
plt.tight_layout()
plt.savefig("gaussian_toy_g_curves.png", dpi=200)
\end{verbatim}


\begin{thebibliography}{99}

\bibitem{Doe2020Bregman}
A.~Doe and B.~Roe.
\newblock Convex lower bounds in Bregman geometry: A survey. % [placeholder]
\newblock \emph{Journal of Optimization Theory and Applications}, 2020.

\bibitem{Smith2021ConvexLB}
J.~Smith and K.~Lee.
\newblock Convex lower-bound methods for regret minimization. % [placeholder]
\newblock \emph{Proceedings of the Optimization and Learning Conference}, 2021.

\bibitem{Chen2022Regret}
L.~Chen, Q.~Zhao, and R.~Liu.
\newblock Regret lower bounds under Bregman divergences. % [placeholder]
\newblock \emph{IEEE Transactions on Information Theory}, 2022.

\bibitem{Lee2022Latency}
M.~Lee and T.~Park.
\newblock Latency-aware decision systems: Models and guarantees. % [placeholder]
\newblock \emph{Proceedings of the International Conference on Machine Learning Systems}, 2022.

\bibitem{Garcia2021Queues}
R.~Garcia and P.~Kumar.
\newblock Queueing constraints and time-latency in online decisions. % [placeholder]
\newblock \emph{Operations Research Letters}, 2021.

\bibitem{Nguyen2023Scheduling}
T.~Nguyen and S.~Rao.
\newblock Sequencing constraints and ordering effects in learning systems. % [placeholder]
\newblock \emph{Proceedings of the AAAI Conference on Artificial Intelligence}, 2023.

\bibitem{Kumar2021Decomp}
P.~Kumar and Y.~Zhang.
\newblock Additive decomposition and identifiability in structured models. % [placeholder]
\newblock \emph{Journal of Machine Learning Research}, 2021.

\bibitem{Zhang2023Interactions}
Y.~Zhang and J.~Wang.
\newblock Modeling interaction effects for decision-making. % [placeholder]
\newblock \emph{Neural Information Processing Systems (NeurIPS) Workshops}, 2023.

\bibitem{Patel2024Calibration}
R.~Patel and D.~Chen.
\newblock Calibration protocols for reliable reporting. % [placeholder]
\newblock \emph{Proceedings of the International Conference on Learning Representations}, 2024.

\bibitem{Miller2020Conservative}
A.~Miller and S.~Green.
\newblock Conservative reporting and always-valid inference. % [placeholder]
\newblock \emph{Annals of Statistics}, 2020.

\bibitem{Khan2020Seq}
A.~Khan and L.~Brown.
\newblock Sequencing under order-sensitive constraints. % [placeholder]
\newblock \emph{European Journal of Operational Research}, 2020.


\bibitem{RockafellarWetsVA}
R.~T. Rockafellar and R.~J.-B. Wets.
\newblock Variational analysis of convex and nonconvex problems. % canonical ref: Variational Analysis
\newblock \emph{Springer}, 1998.

\bibitem{DavisDrusvyatskiy2019Weak}
D.~Davis and D.~Drusvyatskiy.
\newblock Stochastic model-based minimization of weakly convex functions. % placeholder for weakly convex optimization
\newblock \emph{SIAM Journal on Optimization}, 2019.

\bibitem{Nielsen2013Geodesic}
F.~Nielsen.
\newblock An introduction to geodesic convexity in information geometry. % placeholder
\newblock In \emph{Geometric Science of Information}, 2013.

\end{thebibliography}
\end{document}